\documentclass{article}
\usepackage{spconf,amsmath,graphicx}
\usepackage[hidelinks]{hyperref}

\usepackage[utf8]{inputenc} % allow utf-8 input
\usepackage[T1]{fontenc}    % use 8-bit T1 fonts

\usepackage{url}            % simple URL typesetting
\usepackage{booktabs}       % professional-quality tables
\usepackage{amsfonts}       % blackboard math symbols
\usepackage{nicefrac}       % compact symbols for 1/2, etc.
\usepackage{microtype}      % microtypography
\usepackage[ruled]{algorithm2e}
\usepackage{wrapfig}
\usepackage{amstext}

\usepackage{multirow}
\usepackage{graphicx}
\usepackage{float}

\usepackage{amsbsy}
\usepackage{amssymb}
\usepackage{amsmath}
\usepackage{amsthm}
\usepackage{color}
\usepackage{tabularx}
\usepackage{booktabs}
\usepackage{bbm}
\usepackage{textcomp}

\usepackage{tikz,pgf,tikz-3dplot}
\usepackage{tikz}

\newcolumntype{C}[1]{>{\centering\arraybackslash}m{#1}}

\def\SS{{\mathcal {S}}}

\def\L{{\mathbf L}}

\def\R{{\mathbb R}}

\def\E{ \mathbb{E} }
\def\Var{ \text{Var} }

\usepackage{bm}  % in your preamble

\def\y{{\mathbf y}}

\def\b{{\mathbf b}}

\def\x{{\mathbf x}}
\def\f{{\mathbf f}}
\def\z{{\mathbf z}}

\def\SS{ \mathcal{S} }
\def\NN{ \mathcal{N} }

\def\1{\mathbbm{1}}

\newcommand{\tr}[1]{\operatorname{tr}\left(#1\right)}

\def\bomega{ \boldsymbol{\omega} }
\def\bOmega{ \boldsymbol{\Omega} }

\def\b0{ \mathbf{0} }

\newtheorem{corollary}{Corollary}

\newtheorem{theorem}{Theorem}
\newtheorem{remark}{Remark}

\definecolor{salmon}{rgb}{1.0, 0.55, 0.41}
\definecolor{dodgerblue}{rgb}{0.12, 0.56, 1.0}
\definecolor{orchid}{rgb}{0.85, 0.44, 0.84}
\definecolor{limegreen}{rgb}{0.2, 0.8, 0.2}

\usepackage[font=small,labelfont=bf]{caption}
\DeclareCaptionLabelFormat{andtable}{#1~#2  \&  \tablename~\thetable}

\def\x{{\mathbf x}}
\def\L{{\cal L}}

\title{Efficient Gaussian Process Learning via Subspace Projections}
\twoauthors{Elsa Cazelles$^\ast$}
	{CNRS, IRIT, Université de Toulouse}
{Felipe Tobar$^\ast$}
	{Imperial College London}
\begin{document}
\ninept
\maketitle
\def\thefootnote{*}\footnotetext{The authors are listed in alphabetical order.}\def\thefootnote{\arabic{footnote}}

\begin{abstract}
We propose a novel training objective for GPs constructed using lower-dimensional linear projections of the data, referred to as \emph{projected likelihood} (PL). We provide a closed-form expression for the information loss related to the PL and empirically show that it can be reduced with random projections on the unit sphere. We show the superiority of the  PL, in terms of accuracy and computational efficiency, over the exact GP training and the variational free energy approach to sparse GPs over different optimisers, kernels and datasets of moderately large sizes.
\end{abstract}
\section{Introduction}
\label{sec:intro}

%\subsection{Gaussian process regression}
The Gaussian process (GP) \cite{williams1995gaussian} is the \emph{de facto} approach to time series, especially in regimes of unevenly-sampled observations, missing data and when uncertainty quantification is required. GPs have been applied to EEG \cite{caro2022modeling}, audio \cite{alvarado2016gaussian}, seismology \cite{gentile2020gaussian}, climate \cite{andersson2023environmental}, body-motion sensing \cite{wang2007gaussian} and astronomy \cite{covino2022detecting}. What makes GPs suitable for a wide range of time series applications is a comprehensive procedure for designing and choosing tailored covariance kernels, which allow practitioners to incorporate inductive biases, thus elegantly blending prior knowledge with rich relationships learnt from data.

Conceptually, GP regression models construct a Gaussian prior distribution on the infinite-dimensional space of functions, defined by a mean function (usually set to zero) and a covariance kernel \cite{Rasmussen:2006}. Then, observations are considered as parts of the latent true function and used to update the GP prior into a GP posterior. For the model to accurately represent the data, the prior needs to be calibrated with respect to the available data, meaning that particular mean and covariance functions need to be chosen. Denoting the number of observations by $n$, GP training is performed by means of maximum likelihood (ML) at a cost $\mathcal{O}(n^3)$, which renders this approach unfeasible for more than a few thousand datapoints.

A vast collection of approximate-likelihood methods has been developed to alleviate the computational cost related to standard ML-based GP training \cite{quinonero2005unifying,rossi21a}. These approaches either directly replace the inverse covariance matrix with a lower-rank approximation, or assume the conditional independence of the data under a set of (trainable) \emph{inducing variables}, which also results in a reduced-complexity objective. Despite the proven success of the inducing-variable approach to GPs, termed sparse GPs, they are prone to: i) provide biased learnt hyperparameters for a reduced computational cost, and ii) only achieve superior computational efficiency asymptotically (i.e., for $n$ sufficiently large). 

For instance, \cite{bauer2016understanding} states that the fully independent training conditional (FITC) method  \cite{snelson2006sparse} is known to underestimate the noise variance, while the variational free energy (VFE) method \cite{titsias2009variational} is known to overestimate it. For VFE, this hyperparameter mismatch comes from maximising the evidence lower bound, whose bias follows from the restricted choice of the approximating variational distribution \cite{Blei2017variational}. Furthermore, due to the more complex training loss of sparse GPs, involving several Cholesky decompositions, VFE's asymptotic efficiency is only realised for $n$ sufficiently large. 

To address the need for a GP training method that is both computationally efficient and accurate for moderately large datasets, we propose a novel training objective, termed \emph{projected likelihood} (PL), given by the likelihood of a surrogate GP over a lower-dimensional linear projection of the data. We quantify the information loss related to the PL and show it can be reduced using $k$ random projections, thus attaining a training cost $\mathcal{O}(kn^2)$, where $k\ll n$. Despite this cost being seemingly larger than the standard $\mathcal{O}(nm^2)$ cost of the VFE sparse GP with $m$ inducing points, we empirically show that the PL is computationally efficient, due to requiring fewer optimisation steps, and accurate in terms of achieving solutions that are closer to those of the exact GP method. The proposed PL method is experimentally validated using real-world and synthetic data of varying magnitude, as well as different kernels and optimisers.

\section{Gaussian processes for regression}

\subsection{Training and inference}

A \emph{Gaussian process} (GP) is a stochastic process 
$\{ f(x) : x \in \R^d \}$ such that any finite collection 
$(f(x_1), \dots, f(x_n))$ follows a $n$-multivariate Gaussian distribution. 
We adopt the notation
\[
f(x) \sim \mathcal{GP}\big(m(x), k(x,x')\big),
\]
where $m(x) = \mathbb{E}[f(x)]$ is the mean function---usually set to zero---and 
$k(x,x') = \operatorname{Cov}(f(x), f(x'))$ is the covariance kernel.

Given training data $\mathcal{D} = \{(x_i, y_i)\}_{i=1}^n$, assuming a Gaussian likelihood 
$y_i = f(x_i) + \varepsilon$, $\varepsilon \sim \mathcal{N}(0,\sigma^2_{\text n})$, and denoting the observations as stacked vectors $\x$ and $\y$, the marginal likelihood is
\begin{equation}
    p(\y | \x, \theta) 
= \mathcal{N}\!\big(\y| \, m_\x, \, K_{\x} + \sigma^2_{\text n} \mathbf{I}\big),
\label{eq:GP_prior}
\end{equation}
where $\theta$ are the kernel hyperparameters, and we used the compact notation $m_\x=m(\x)=[m(x_1),\ldots,m(x_n)]^T$ and $K_{\x} = K(\x,\x)$, where $(K(\x,\x))_{ij} = k(x_i,x_j)$. We do not denote the explicit dependence of $m$ and $k$ on $\theta$ for the sake of notation simplicity.

Training is performed by minimising the negative log likelihood (NLL) $\L (\theta) = -\log p(\y  |\x, \theta) $ given by (assuming $m_\x=0$):
\begin{equation}
    \L (\theta) = \tfrac{1}{2}\y^\top (K_{\x} + \sigma^2_{\text n}\mathbf{I})^{-1}\y
   + \tfrac{1}{2}\log |K_{\x} + \sigma^2_{\text n}\mathbf{I}| + \frac{n}{2}\log 2\pi.\label{eq:nll}
\end{equation}

Furthermore, for test inputs $\x_*$, the GP's predictive posterior distribution is given by \cite{Rasmussen:2006} $p(\f_* | \x, \y, \x_*) = 
\mathcal{N}\!\big(\mu_*, \Sigma_*\big)$, with 
\begin{align}
\mu_* &= K_{\x_*\x}(K_{\x}+\sigma^2_{\text n}\mathbf{I})^{-1}\y, \\
\Sigma_* &= K_{\x_*} -
            K_{\x_*\x}\,(K_{\x}+\sigma^2_{\text n} \mathbf{I})^{-1}K_{\x\x_*},
\end{align}
where $K_{\x_*} = K(\x_*,\x_*)$ and $K_{\x_*\x}^\top = K_{\x\x_*} = K(\x,\x_*)$.

\subsection{Sparse approximations}

To alleviate the $\mathcal{O}(n^3)$ cost of eq.~\eqref{eq:nll}, computationally-efficient approximations of $\L(\theta)$ rely on $m$ \emph{inducing variables}, which summarise the observations sufficiently well \cite{rossi21a} . Specifically, the reduced computational cost comes from assuming that the $n$ observations are conditionally independent given the $m$ inducing variables. This approach, termed  \emph{sparse GPs}, includes the early Deterministic Training
Conditional (DTC) \cite{seeger03a} and the Fully Independent Training Conditional (FITC) \cite{snelson2006sparse}. For a revision of classic sparse GP methods and how they relate to one another, see \cite{quinonero2005unifying}. 

A cornerstone in sparse GP approximations is the Variational Free Energy (VFE) method \cite{titsias2009variational}, which has become the basis for subsequent approaches. VFE models the inducing variables by a variational distribution $q(\mathbf{u})$, thus approximating the posterior over $f$ as $q(f) = \int p(f | \mathbf{u}) q(\mathbf{u}) d\mathbf{u}$. Training then involves finding the variational parameters by maximising the evidence lower bound (ELBO):
\begin{equation}
    \mathcal{L} = \mathbb{E}_{q(f)}[\log p(\y | f)] - \mathrm{KL}[q(\mathbf{u}) \| p(\mathbf{u})],
\end{equation}
which ensures the approximation is close to the true posterior while reducing complexity to $\mathcal{O}(nm^2)$. Denoting the locations of the inducing variables by $\bar\x\in\R^{d\times m}$, the ELBO is given by 
\begin{align}
\mathcal{L}_{\mathrm{VFE}} &=
-\frac{1}{2} \Big[ 
\y^\top (Q_{\x} + \sigma^2_{\text n} \mathbf{I})^{-1} \y 
+ \log |Q_{\x} + \sigma^2_{\text n} \mathbf{I}| 
+ n \log 2\pi
\Big]\nonumber \\
&- \frac{1}{2\sigma^2_{\text n}} \mathrm{tr}(K_{\x} - Q_{\x}),
\end{align}
where $Q_{\x} = K_{\x\bar\x} K_{\bar\x}^{-1} K_{\bar\x\x}, K_{\bar\x} = k(\bar\x,\bar\x)$ and $K_{\x\bar\x} = k(\x,\bar\x)$.

Since $\bar\x$ has to be learnt, the introduction of additional variational parameters leads to learning biased hyperparameters and a computational efficiency that only arises for sufficiently large $n$.

\section{The projected likelihood method}

\subsection{Projections onto the unit sphere $\SS^{n-1}$}

We consider $k\ll n$ linear projections of the observations, interpreted as a $k$-dimensional, lossy compression of the $n$-dimensional data distribution, from which we can derive a cheaper training objective approximating $\L(\theta)$ in eq.~\eqref{eq:nll}. We conjecture that, if the relevant directions of $K_{\x}$ are preserved in the projected data, the hyperparameters can be recovered with minimal information loss and computational overhead.

We project $\y\in\R^n$ onto a set of directions $\bomega_1,\ldots,\bomega_k\in\SS^{n-1}$, where $\SS^{n-1}$ denotes the unit sphere in $\R^n$, obtaining
\begin{equation} 
\label{eq:projections} \z = [\bomega_1^\top\y,\ldots,\bomega_k^\top\y]\in\R^k.
\end{equation}
Assuming $\y$ is distributed by a zero-mean multivariate normal (MVN) distribution with a positive-definite invertible covariance matrix $K_{\x} +\sigma^2_{\text n}\mathbf{I}$, then $\z$ in eq.~\eqref{eq:projections} follows a $k$-dimensional MVN given by
\begin{equation}
    p(\z\vert \x, \theta) = \NN(\z \vert 0, \bOmega^\top(K_{\x} +\sigma^2_{\text n}\mathbf{I})\bOmega),
\label{eq:GP_pdf_projected}
\end{equation}
where $\bOmega\in\R^{n\times k} = [\bomega_1,\bomega_2,\ldots,\bomega_k]$ is the projection matrix.
\begin{remark}
For  eq.~\eqref{eq:GP_pdf_projected} to be non-degenerate, it is required that $\text{rank}(\bOmega)=k$, or equivalently, the projection vectors $\{\bomega_1,\ldots,\bomega_k\}\subset\SS^{n-1}$ have to be linearly independent.
\end{remark}

Our training objective will be the (negative logarithm) of the \emph{projected likelihood} (PL), which has a cost $\mathcal{O}(kn^2)$, and is given by 
\begin{align}
\L_{\text{PL}}(\theta) 
&= \frac{1}{2} \Bigg[ 
\z^\top \left( \mathbf{\Omega}^\top 
(K_{\x} + \sigma^2_{\text n} \mathbf{I})^{-1} 
\mathbf{\Omega}\right)^{-1} \z \nonumber\\
&\quad + \log\left| \mathbf{\Omega}^\top 
(K_{\x} + \sigma^2_{\text n} \mathbf{I})^{-1} 
\mathbf{\Omega}\right| 
+ k \log 2\pi
\Bigg]. \label{eq:PL_NLL}
\end{align}

\subsection{Information loss in the projected likelihood}
\label{sec:inf_loss}

The following result justifies PL as a loss for learning a GP. 

\begin{theorem}
    Denote $Y$ the random variable of the full GP evaluated on $\x$, and $Z=\bOmega^\top Y$ its projection. The Fisher information loss related to learning $\theta$ through the PL in eq.~\eqref{eq:PL_NLL}, instead of the exact NLL in eq.~\eqref{eq:nll}, is quadratic in $\Sigma_{Y|Z}=\text{Cov}(Y|Z)$.
\end{theorem}

\begin{proof}
We denote $\y$ and $\z$ observations of $Y$ and $Z$, as in the previous section. We consider a scalar kernel parameter $\theta\in\R$, as the proof for multivariate parameters follows componentwise. Recall that the score and Fisher information of $Y$ are respectively $s(\y) :=\partial_\theta \log p(\y)$ and $I_Y(\theta) = \Var(s(Y))$; and equivalently for $s(\z)$ and $I_Z$. By the change of variable theorem $p(\z | \theta) = \int p(\z | \theta) \textbf{1}_{\{\bOmega^\top \y = \z\}}d\y$
and the definition of the score, we have 
\begin{equation}
    s(\z) = \E_Y[s(Y) | Z = \z].
\end{equation}
Furthermore, applying the definition of the Fisher information and the law of total variances, we have the known result \cite{cover2006elements}: 
\begin{equation}
    I_Y(\theta) - I_Z(\theta) = \E_Z[\Var(s(Y)|Z)] \geq 0,
    \label{eq:deltaI}
\end{equation}
which is a general expression for the information lost when learning the parameter $\theta$ from $Z$ rather than $Y$. Denoting $K_\x$ the covariance including the noise term for simplicity, $Y\sim\NN(0,K_\x)$ gives
\begin{equation}
    s(Y) = - \frac{1}{2}\tr{K_{\x}^{-1}\dot K_{\x}} + \frac{1}{2}Y^\top K_{\x}^{-1}\dot K_{\x}K_{\x}^{-1}Y,
    \label{eq:score_Y}
\end{equation}
where $\dot K_{\x} = \partial_\theta K_{\x}$.
To compute $\Var (s(Y)| Z)$, let us first recall that $Y|Z \sim \NN(\mu_{Y|Z}, \Sigma_{Y|Z})$, where 
\begin{align}
    \mu_{Y|Z} &= K_{\x} \bOmega(\bOmega^\top K_{\x} \bOmega)^{-1} Z \label{eq:meanYZ}\\
    \Sigma_{Y|Z} &= K_{\x} - K_{\x} \bOmega(\bOmega^\top K_{\x} \bOmega)^{-1} \bOmega^\top K_{\x},\label{eq:VarYZ}
\end{align}
and that $s(Y)$ in eq.~\eqref{eq:score_Y} has a constant first term and a quadratic term for which the variance can be calculated explicitly. Therefore, 
\begin{align}
    \Var (s(Y)|Z) 
    &= 
    \frac{1}{2}\tr{(K_{\x}^{-1}\dot K_{\x}K_{\x}^{-1}\Sigma_{Y|Z})^2}\nonumber\\
    &+
    \mu_{Y|Z}^\top K_{\x}^{-1}\dot K_{\x}K_{\x}^{-1} \Sigma_{Y|Z} K_{\x}^{-1}\dot K_{\x}K_{\x}^{-1}\mu_{Y|Z}.
\end{align}
Rearranging to express both terms as a trace, and applying the expectation wrt $Y$, we have from eq.~\eqref{eq:deltaI} that $\Delta I := \E_Z[\Var(s(Y)|Z)]$ with:
\begin{align}
   \Delta I 
    &= 
    \tr{\frac{1}{2} (\bar K_{\x} \Sigma_{Y|Z})^2 + \bar K_{\x} \Sigma_{Y|Z} \bar K_{\x}  \E[\mu_{Y|Z}\mu_{Y|Z}^\top]},\label{eq:inf_loss}
\end{align}
where $\E_Z[\mu_{Y|Z}\mu_{Y|Z}^\top] = K_{\x}\bOmega (\bOmega^\top K_{\x}\bOmega)^{-1} \bOmega^\top K_{\x}$, and we denoted $\bar K_{\x} = K_{\x}^{-1}\dot K_{\x}K_{\x}^{-1}$. This gives the explicit dependence of $\Delta I$ on $(\Sigma_{Y|Z})^2$ in eq.~\eqref{eq:VarYZ}, thus concluding the proof.
\end{proof}

\begin{corollary}
    If $k=n$ linearly independent projections are chosen, there is no information loss. 
\end{corollary}
\begin{proof}
    In this case, the map between $Y$ and $Z$ is bijective. Therefore, $\bOmega$ is invertible and $\Sigma_{Y|Z}$ vanishes, implying $\Delta I=0$ in eq.~\eqref{eq:inf_loss}.
\end{proof}

\begin{corollary}
For a white noise setting $p(\y | \x, \theta) 
= \mathcal{N}\!\big(\y| \, 0, \, \sigma^2_{\text n} \mathbf{I}\big)$, and $\bOmega:=\bomega\in\SS^{n-1}$, the PL loss is equivalent to the NLL loss.
\end{corollary}
\begin{proof}
    From eq.~\eqref{eq:GP_pdf_projected}, we have $p(\z | \x, \theta)  = \mathcal{N}\!\big(\z| \, 0,\bomega^T\sigma^2_{\text n} \mathbf{I}\bomega\big) = \mathcal{N}\!\big(\z| \, 0,\sigma^2_{\text n}\big)$, and therefore $\mathcal{L}(\theta) = n\mathcal{L}_{\text{PL}}(\theta)$.
\end{proof}

Since our focus is on $k\ll n$, notice that $\Sigma_{Y|Z}$ is the residual of a $k$-rank approximation of the $n$-rank covariance $K_{\x}$. Therefore, the optimal projection directions are the $k$-top eigenvectors of $K_\x$---which are unknown before training. We explore this in detail and provide a criterion for choosing $\bOmega$ justified experimentally in Sec.\ref{sec:choose_proj}.

\subsection{Interpretation and relationship to previous methods}

A direct connection between the PL and sparse GPs using inducing variables is that the projected variables $\z\in\R^k$ can be considered as interdomain inducing variables \cite{NIPS2009_5ea1649a}. Therefore, PL assumes that the GP is sparse in the transformed domain since, beyond a certain value of $n$, additional observations provide minimal information about $\theta$ and can be discarded to pursue computational efficiency. Sparse GP methods achieve this by constructing \emph{trainable} interdomain/temporal inducing variables, which involves an extra computational overhead. Conversely, PL relies on a fixed set of directions, thus avoiding the training of additional optimisation variables. As a result, and as verified by our experiments, PL exhibits a reduced computational cost in reasonably large datasets compared to VFE in practice, despite its theoretical computational complexity. Furthermore, PL achieves an NLL that is closer to that of the exact solution. 

\section{Design of projection directions}
\label{sec:choose_proj}
We aim to choose the projections to minimise the information loss in Sec.~\ref{sec:inf_loss}. Since both terms in eq.~\eqref{eq:inf_loss} are positive definite matrices, $\Delta I$ can be reduced by minimising the norm of $\Sigma_{Y|Z}$ in eq.~\eqref{eq:VarYZ}, or equivalently, its trace. 

Denote the eigendecomposition $K_\x=U\Lambda U^\top$, with $U\in\R^{n\times n}$ an orthogonal eigenvectors matrix, and $\Lambda = \text{diag}(\lambda_1,\ldots,\lambda_n)$ are the eigenvalues. Since any choice of a $k$-rank  projection matrix will span a $k$-dimensional subset of the span of $U$, the projection matrix can be expressed without loss of generality as the choice of $k$ eigenvectors of $K_\x$, denoted by $\bOmega = U_k\in\R^{n\times k}$, with corresponding diagonal eigenvalue matrix $\Lambda_k$. Plugging this into eq.~\eqref{eq:VarYZ}, we have 
\begin{align}
    \Sigma_{Y|Z} &= 
    U \Lambda U^\top - U \Lambda U^\top U_k(U_k^\top U \Lambda U^\top U_k)^{-1} U_k^\top U \Lambda U^\top,\nonumber \\
    &= 
    U (\Lambda - \Lambda_k)U,
\end{align}  
where the last step follows from the product $U^\top U_k$ giving a masked identity matrix containing ones only at the coordinates corresponding to the chosen eigenvalues. Then for any $A\subset \{\lambda_1,\ldots,\lambda_n\}$,  
\begin{equation}
    \min_{|A| = k} \sum_{\lambda \in A} \lambda  \leq  \tr{\Sigma_{Y|Z}} \leq   \max_{|A| = k} \sum_{\lambda \in A} \lambda,
\end{equation}
meaning that the norm of $\Sigma_{Y|Z}$ is bounded by the worst ($k$ smallest eigenvalues) and best ($k$ largest eigenvalues) choice of projection directions given by the eigenvectors of $K_\x$. Since realising this (best) choice is impossible in practice since $K_\x$ is updated during training, we experimentally show that several random projections with different structures effectively reduce $\tr{\Sigma_{Y|Z}}$.

Let us consider the following orthogonal projections: 

\begin{itemize}
    \item \textbf{Sphere:} $\bomega_i$ uniformly distributed on $\SS^{n-1}$
    \item \textbf{Repulsive:} as above, with pairwise repulsion
    \item \textbf{Localised:} Gaussian RBFs covering the range of observations
    \item \textbf{One hot:} each projection contains only one observation
\end{itemize}

We constructed $K_\x$ as the Gram matrix of a square exponential (SE) kernel ($\sigma^2=5, l = 10, \sigma_{\text n}^2 = 0.1$) evaluated on $n=2000$ points. Fig.~\ref{fig:conditional_var} shows the spectrums of $K_\x$ (gold) and $\Sigma_{Y|Z}$ (blue) for $k\in\{50,100,200\}$. Observe how all the projection structures critically reduce the trace of $\Sigma_{Y|Z}$, with the spherical projection providing the best reduction and the naive one-hot projection providing the worst one due to capturing fewer datapoints. Consequently, we will consider spherical projections in the implementation of PL. 
\begin{figure}
    \centering
    \includegraphics[width=0.49\linewidth]{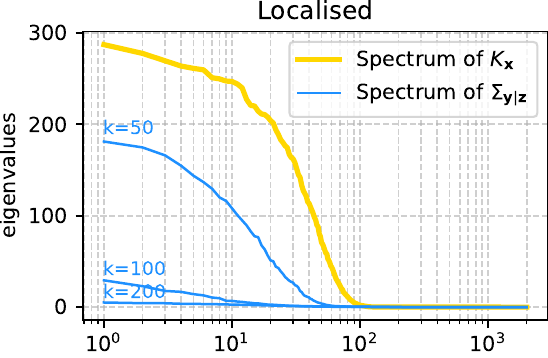}
    \includegraphics[width=0.49\linewidth]{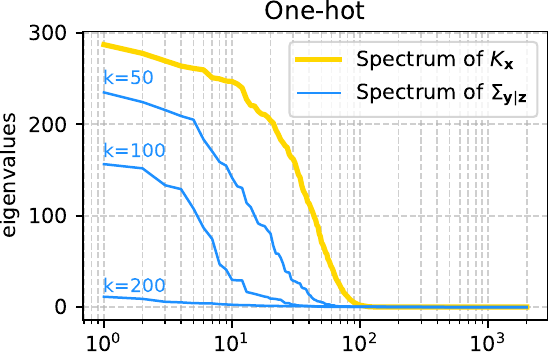}
    \includegraphics[width=0.49\linewidth]{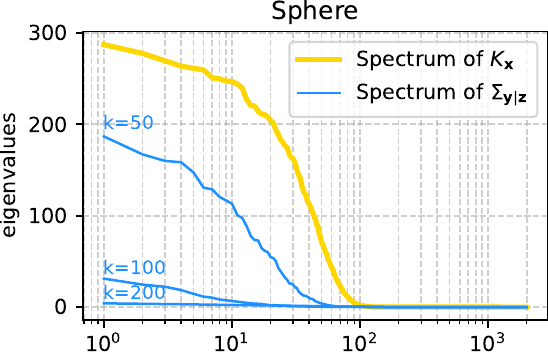}
    \includegraphics[width=0.49\linewidth]{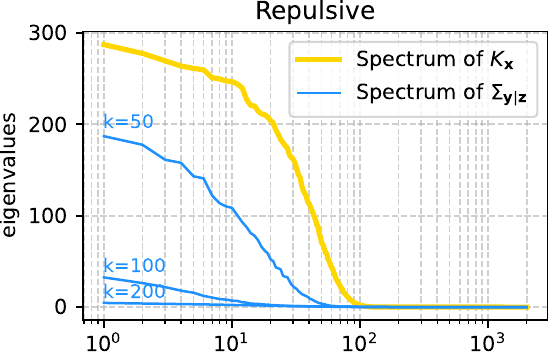}
    \caption{Eigenspectrum of $K_\x$ and $\Sigma_{Y|Z}$ using different types and number of projection matrices $\bOmega$.}
    \label{fig:conditional_var}
\end{figure}

\section{Experiments}

We validated the proposed PL loss in eq.~\eqref{eq:PL_NLL} with spherical projection against VFE in three experiments. All simulations were executed on a Macbook Pro (M3 processor) using only the CPU. Code available here:  \href{https://github.com/felipe-tobar/projected-likelihood}{\textbf{github.com/felipe-tobar/projected-likelihood}}.

\subsection{E1: Different optimisers and variance assessment}

We trained VFE and PL, with their respective parameter $k=m=100$, on a 1000-sample dataset generated by a GP with SE kernel ($\sigma^2=1, l = 20, \sigma_{\text n}^2 = 0.1$). We considered BFGS (50 iters, lr = $5\cdot10^{-3}$) and Adam \cite{adam} (2000 iters, lr= $1\cdot10^{-1}$) optimisers; the stopping criteria for both was to have 5 iterations with less than $1\cdot10^{-2}$ improvement. Table \ref{tab:exp1} shows the NLL and running time. In all cases, PL achieved an NLL  closer to the exact GP baseline than VFE, while also being faster due to the reduced number of iterations required. Given the better performance of Adam, we chose this optimiser for the remaining experiments. 

\begin{table}[ht]
\scriptsize
\centering
\caption{NLL and running time for different methods and optimisers}
\begin{tabular}{lccc|ccc}
\toprule
& \multicolumn{3}{c|}{BFGS} & \multicolumn{3}{c}{Adam} \\
\cmidrule(lr){2-4} \cmidrule(lr){5-7}
Metric & ML & VFE & PL & ML & VFE & PL \\
\midrule
NLL ($\downarrow$)  & 337.4 & 358.5 & \textbf{349.4} & 337.4 & 358.1 & \textbf{347.0} \\
Time [s] ($\downarrow$) & 26.5 & 4.95 & \textbf{3.76} & 3.66 & 1.57 & \textbf{0.574} \\
\bottomrule
\end{tabular}
\label{tab:exp1}
\end{table}

Additionally, Fig.~\ref{fig:exp1} shows the posterior variance of the GPs learnt by the considered methods. In line with previous findings \cite{bauer2016understanding}, VFE overestimated the posterior variance of the process, while PL's variance remained closer to the ML estimate.

\begin{figure} 
    \centering 
    
    \includegraphics[height=0.49\linewidth]{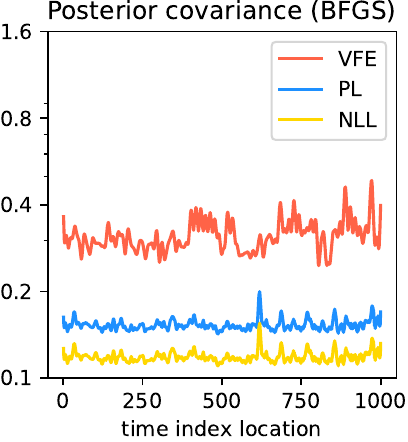}
    \includegraphics[height=0.49\linewidth]{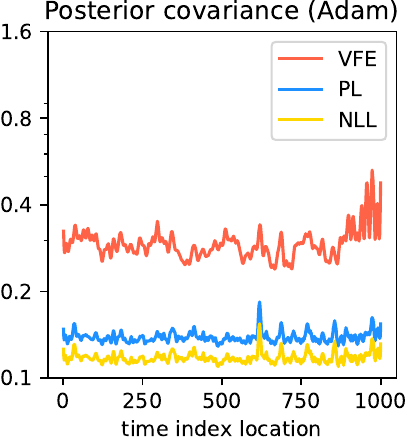}
    \caption{Posterior variance of the GP learnt by ML, VFE and proposed PL methods, using the Adam and BFGS optimisers.}
    \label{fig:exp1}
\end{figure}

\subsection{E2: Effect of the number of projections}

We then compared PL and VFE for an increasing number of observations, with both $k$ (PL) and $m$ (VFE) in the set ${50,100,150}$. Fig.~\ref{fig:exp2} shows the achieved NLL ($y$-axis) vs the computation time ($x$-axis) in log-log scale for six experiments with observations $n\in\{500,1000,1500,2000,3000,4000\}$. In each plot, VFE (red) and PL (blue) are compared to the exact ML solution (yellow). 

In all cases, both VFE and PL achieved the same NLL as the exact ML method for $m=k=150$. Critically, PL achieved the desired NLL quicker than VFE for $n<3000$, a difference that was gradually reduced as $n$ increased, where both methods performed on par in accuracy and computational cost. This is a consequence of PL's quadratic cost and VFE's linear cost, which eventually coincide for growing $n$. Accordingly, this experiment confirms the superiority of PL in moderately-sized datasets.

\begin{figure}[t]
    \centering 
    \includegraphics[width=1\linewidth]{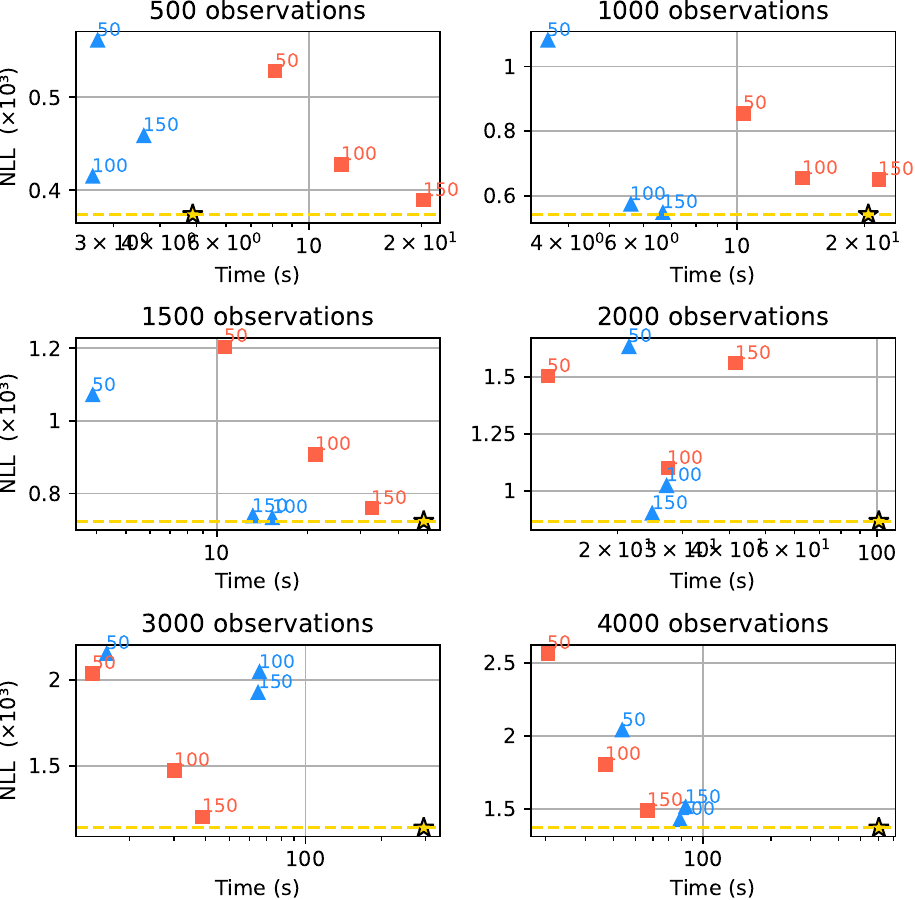}
    \caption{Performance [NLL] versus computation time [seconds] for VFE (red squares) and PL (blue triangles). The orders $m$ (VFE) and $k$ (PL) are denoted in each marker. The ML solution is denoted with a yellow star, with the achieved NLL in a dashed yellow line. The closer to the bottom-left corner, the better.}
    \label{fig:exp2}
\end{figure}

\subsection{E3: Real-world series and different kernels}

\begin{table}[ht]
\centering
\scriptsize
\caption{Sunspots (3319 datapoints): Performance and running time }
\begin{tabular}{lcccccc}
\toprule
Kernel & \multicolumn{2}{c}{ML} & \multicolumn{2}{c}{VFE ($m=100$)} & \multicolumn{2}{c}{PL ($k=100$)} \\
\cmidrule(lr){2-3} \cmidrule(lr){4-5} \cmidrule(lr){6-7}
 & NLL & Time & NLL & Time & NLL & Time \\
\midrule
SE & 1549.120 & 197.15 & 1691.999 & \textbf{6.78} & \textbf{1584.565} & 19.62 \\
Laplace & 1432.688 & 172.38 & 1944.358 & 69.06 & \textbf{1675.062} & \textbf{31.46} \\
RQ & 1531.909 & 212.07 & 1765.050 & \textbf{24.18} & \textbf{1588.746} & 42.01 \\
LocPer & 1510.357 & 231.84 & 1769.835 & 36.89 & \textbf{1547.013} & \textbf{31.26} \\
\bottomrule
\end{tabular}
\label{tab:sunspots}
\end{table}

\begin{table}[H]
\centering
\scriptsize
\caption{EEG: Performance and running time (8000 datapoints)}
\begin{tabular}{lcccc}
\toprule
Kernel & \multicolumn{2}{c}{VFE} & \multicolumn{2}{c}{PL} \\
\cmidrule(lr){2-3} \cmidrule(lr){4-5}
 & RMSE & Time (s) & RMSE & Time (s) \\
\midrule
SE & 0.303 & \textbf{148.08} & \textbf{0.194} & 306.22 \\
Laplace & 0.270 & \textbf{614.32} & \textbf{0.145} & 630.61 \\
RQ & 0.285 & 402.32 & \textbf{0.192} & \textbf{337.68} \\
LocPer & 0.289 & 611.61 & \textbf{0.170} & \textbf{554.79} \\
\bottomrule
\end{tabular}
\label{tab:eeg}
\end{table}

We then considered two real-world datasets: the 3319-sample monthly sunspot series \cite{SILSO_Sunspot_Number} and a 10000-sample subset of the Helsinki neonatal EEG dataset, \cite{stevenson2018neonatal}. We implemented VFE and PL to train GPs on two real-world datasets using four different kernels: SE, Laplace, Rational Quadratic (RQ) and Locally Periodic (LocPer). For the sunspots, the baseline is the NLL achieved by exact ML, while for the EEG we provide the relative mean square error RNMSE (8000 training points and 2000 validation points), since exact GP training on 10000 observations was unfeasible. 

Tables \ref{tab:sunspots} and \ref{tab:eeg} show the performance of all methods considered. Even though a similarity in computational cost between VFE and PL arises for these large datasets, the superior accuracy of PL (either in NLL or RMSE) is confirmed for all datasets and kernels considered.

\section{Conclusions}

    We have proposed the PL method for computationally-efficient GP learning via low-dimensional projections of the data, and have quantified its information loss wrt to exact training. Despite PL's complexity being $\mathcal{O}(nk^2)$, as opposed to sparse GP's $\mathcal{O}(nm^2)$ cost, we have shown that the PL is competitive against VFE in datasets up to 8000 training points, while effectively providing lower NLL or RMSE. We observed that PL's computational gain came from requiring fewer optimisation steps (similar to the exact GP), unlike sparse GP, which computes a cheaper objective but requires more iterations. This calls for assessing approximate-likelihood methods not only in terms of the evaluation of the loss, but also in the complexity of their loss landscape.  
    Conceptually, our study validates the advantage of compressed data representations in GP learning, where natural redundancies in large-scale time series are discarded (see also \cite{bandedCOV}), thus reducing computational cost while minimally sacrificing learning accuracy.

\clearpage

\section{Acknowledgment}
This work is partially supported by the AI cluster ANITI (ANR-23-IACL-0002) and by the France 2030 program (ANR-23-PEIA-0004).

\bibliographystyle{IEEEbib}
\bibliography{references}
\end{document}